\documentclass[letterpaper, 10 pt, conference]{ieeeconf}  % Comment this line out if you need a4paper

\usepackage[T1]{fontenc}
\usepackage{graphicx}
\usepackage{xcolor}
\usepackage{amsmath}
\usepackage{amssymb}
\usepackage{algorithm}
\usepackage[noend]{algpseudocode}
\usepackage{xspace}
\usepackage{tabularx}
\usepackage{booktabs} % To thicken table lines
\usepackage{multirow}
\usepackage{listings}
\usepackage{caption}
\usepackage{subcaption}
\usepackage{hyperref}
\usepackage{wrapfig}
\usepackage{thmtools,thm-restate}
\usepackage{graphicx}
\usepackage{accents}
\usepackage{yhmath}

\newcommand{\vtc}{{SVT}\xspace}

\newcommand{\maxv}{{$v_{max}$}\xspace}
\newcommand{\scenario}{{SC}\xspace}

\newcommand{\mbd}{{$d_{max}$}\xspace}
\newtheorem{theorem}{Theorem}
\newtheorem{definition}{Definition}

\IEEEoverridecommandlockouts                              % This command is only needed if 
\title{\LARGE \bf
Visual Tracking with Intermittent Visibility: Switched Control Design and Implementation
}

\author{Yangge Li$^{1}$, Benjamin C Yang$^{2}$ and Sayan Mitra$^{3}$% <-this % stops a space
\thanks{$^{1}$Yangge Li, University of Illinois Urbana-Champaign
        {\tt\small li213@illinois.edu}}%
\thanks{$^{2}$ Benjamin C Yang, University of Illinois Urbana-Champaign
        {\tt\small bcyang2@illinois.edu}}%
\thanks{$^{3}$ Sayan Mitra, University of Illinois Urbana-Champaign
        {\tt\small mitras@illinois.edu}}%
}

\begin{document}

\maketitle
\thispagestyle{empty}
\pagestyle{empty}

%%%%%%%%%%%%%%%%%%%%%%%%%%%%%%%%%%%%%%%%%%%%%%%%%%%%%%%%%%%%%%%%%%%%%%%%%%%%%%%%
\begin{abstract}

This paper addresses the problem of visual target tracking in scenarios where a pursuer may experience intermittent loss of visibility of the target. The design of a  \textit{Switched Visual Tracker} (\vtc) is presented which aims to meet the competing requirements of maintaining both proximity and  visibility. \vtc alternates between a visual tracking mode for following the target, and a recovery mode for regaining visual contact when the target falls out of sight. We establish the stability of \vtc by extending the average dwell time theorem  from switched systems theory, which  may be of independent interest.
%and helps distill a {\em recoverability assumption} under which tracking performance is assured.
Our implementation of \vtc on an Agilicious drone~\cite{Foehn22science} illustrates its effectiveness on tracking various target trajectories: it  reduces the average tracking error by up to 45\% and significantly improves  visibility duration compared to a baseline algorithm. 
The results show that our approach effectively handles intermittent vision loss, offering enhanced robustness and adaptability for real-world autonomous missions.
Additionally, we demonstrate how the stability analysis provides valuable guidance for selecting 
parameters, such as tracking speed and recovery distance, to optimize the \vtc’s performance.

% The \vtc is implemented on an Agilicious~\cite{Foehn22science} drone and the various design parameters are ...
% to evaluate its effectiveness. Compared to a baseline vision-based tracking algorithm, the \vtc reduces the average tracking error between the target and the pursuer by up to 45\% and increases the fraction of time the target remains visible by up to 8.3$\times$. 
\end{abstract}

%%%%%%%%%%%%%%%%%%%%%%%%%%%%%%%%%%%%%%%%%%%%%%%%%%%%%%%%%%%%%%%%%%%%%%%%%%%%%%%%
\section{Introduction}
% Target tracking plays a crucial role in many autonomous missions. Over the past decades, the problem has been widely studied and applied to various fields, including surveillance, autonomous driving, robotics, and unmanned aerial vehicles (UAVs). On the other hand, vision-based perception has been extensively researched for pose estimation~\cite{liu2024survey} and localization~\cite{9438708}. The low cost and versatility of vision-based algorithms make them a popular choice for enabling modern autonomous systems to react and adapt to diverse tasks.
% Visual target tracking is essential for many autonomous missions, including surveillance, search and rescue, autonomous navigation, and environmental monitoring. Numerous efforts have been made to address this problem, such as in \cite{6086794,9561948,Gomez-Balderas2013,5980372}. However, many of these works do not consider the challenge of intermittent loss of target visibility or fail to provide guarantees on tracking stability through rigorous stability analysis. 

% In this paper, we address the problem of visual target tracking, which involves a \textit{pursuer} and a \textit{target} moving in a shared environment. The objective of the pursuer is to follow the target while maintaining a fixed, pre-specified offset. This task is challenging due to potential scenarios where the pursuer may temporarily lose sight of the target, making continuous tracking difficult.

The visual tracking task
requires a \textit{pursuer} to follow a moving \textit{target} using only camera as a sensor. This capability is relevant for  search and rescue, delivery, spacecraft docking, in-air-refuelling, navigation, and other autonomous missions. 
Visual tracking, also called visual pursuit,
has been studied by robotics and aerospace researchers~\cite{6564808,6842280,6086794,8205986} (see, other related works in Section~\ref{sec:related}). A popular approach is for the pursuer's controller to minimize the  tracking error defined on image space.
% uses the offset of the target from the center of the image frame to maintain tracking. 
% \sayan{odd to single out one work and then ``These methods''}
These methods are effective if the  target is always visible, but cannot recover once it is  lost from the pursuer's camera view. 
%
%
%Several approaches have been proposed to address visual tracking, which also sometimes called visual pursuit (e.g., \cite{6564808,6842280}), where controllers are designed utilizing the offset of the target from the center of the image frame to maintain tracking.
The target may be lost because of motion blur, occlusions, or simply because it moves out of the pursuer's camera frame. The latter is more likely as the pursuer nears the target and the camera's viewable field becomes narrower. Thus,  the two goals of maintaining visibility and gaining proximity  can be conflicting.

% \sayan{say how the assumption may break and why that makes the problem challencing: conflicting requirements.}

%many fail to consider the challenge of intermittent loss of target visibility, which complicates continuous tracking and degrades overall system performance.

To tackle this challenge, we propose the design of a \textit{Switched Visual Tracker} (\vtc)---a mode-switching controller~\cite{DL2003} that tracks the target when it is visible, in what we call a {\em visual tracking mode\/}, and maneuvers to regain visual contact when the target is lost, in a {\em recovery mode}.
The design of \vtc includes the logic for switching between these two modes. A switching controller for landing a drone on a moving vehicle was presented  in~\cite{Gomez-Balderas2013}. That work, like ours, handles loss of visual contact of the moving target and showed interesting empirical results. To our knowledge, ours is the first to provide  a stability analysis and connect the stability criteria with the control design parameter.  
%
% We connect our design to switched system theory and provide a comprehensive stability analysis, a significant contribution as it represents one of the first stability analyses performed on a visual tracking controller.

There is extensive research on stability analysis of switched systems~\cite{DL2003,hespanha99stability,teel2012,DingT10}. Hespanha and Morse's {\em average dwell time\/} theorem~\cite{hespanha99stability} gives a stability criterion in terms of the rate of energy (or Lyapunov function) decay  in the individual modes ($\lambda$), the energy gains across the mode switches, and the  rate of mode switches.
To accommodate the analysis of \vtc with recovery, we generalize this theorem to Theorem~\ref{thm:hybrid_stability} which allows the system to be temporarily unstable, which  leads to an additive ($c$) and multiplicative ($\mu$) increase in the Lyapunov functions.
We show that, given a sufficiently long average dwell time {\em in the stable modes}, the system can still achieve asymptotic stability with respect to a set of states. For \vtc this implies guaranteed tracking performance.
% In this paper, we connect our design to switched system theorem by generalizing this result to include unstable modes and show that with long enough average dwell time, the system is aymptotically stable with respect to a set of states. 
% We generalize an existing result from switched system theory~\cite{} and is able to claim that  
% We implemented the proposed \vtc on an Agilicious drone~\cite{Foehn22science}, with one acting as the target and the other as the pursuer, to evaluate the effectiveness of the algorithm.
% By comparing our \vtc with a basic vision-based tracking algorithm, we demonstrate the advantages of our approach. In the best-case scenario, the average tracking error between the target and the pursuer was reduced by 45\%, and the fraction of time the target remained visible to the pursuer increased by 8.3$\times$ with the \vtc.
% We further assess the performance of the \vtc under various parameter settings to explore its robustness and adaptability.

% We test the effectiveness of the \vtc on our custom hardware platform, which uses the Agilicious platform~\cite{Foehn22science} as the pursuer drone to compare the  baseline visual tracking controller against \vtc. The \vtc improves  the average  tracking error by 45\%  and also significantly improves the  fraction of time the target is visible. 

We compare the effectiveness of \vtc  implemented on the Agilicious drone~\cite{Foehn22science}, with a  baseline visual tracking controller. \vtc improves  the average  tracking error by 45\%  and also significantly improves the  fraction of time the target is visible. 
%In our experiments, the \vtc reduced the average tracking error up to 45\% and improve the fraction of time significantly. 
% \yangge{Additionally, we systematically explored the performance of the \vtc by varying the convergence rate of the Lyapunov function in the stable mode, the upper bound of its increase in the unstable mode, and the average dwell time. Our findings indicate that the tracking performance improves with a higher convergence rate, a longer average dwell time, or a lower upper bound on the increase of the Lyapunov function in the unstable mode, which is consistent with the stability analysis. }
%\yangge{
Further, we observe Theorem~\ref{thm:hybrid_stability}  can be used to guide the choice of  various \vtc parameters for improving the system's performance with respect to tracking and visibility requirements.
% For example, according to Theorem~\ref{thm:hybrid_stability}, a higher tracking speed (\maxv) increases the Lyapunov exponent $\lambda$, which improves tracking by allowing smaller dwell time, and we observe from experiments that increasing \maxv from 0.1 to 1.0, indeed improves the average tracking error from 1.3 to 0.5m. A smaller recovery distance (\mbd) reduces $\mu$ and $c$ (characterizing  the unstable recovery), and from experiments we observe that reducing it from 2.1 to 1.3m, the average tracking error improves  0.63 to 0.48m while target visibility  improves from from 84.5 to 87.4\%. 
For example, according to Theorem~\ref{thm:hybrid_stability}, a higher tracking speed (\maxv) increases the Lyapunov exponent $\lambda$, which reduces tracking  by allowing smaller dwell time. We observe from experiments that increasing \maxv from 0.1m/s to 1.0m/s indeed improves the average tracking error from 1.3m to 0.5m. A smaller recovery distance (\mbd) reduces $\mu$ and $c$ (characterizing  the unstable recovery). By reducing the recovery distance from 2.7m to 2.1m, experiments show the average tracking error improves  0.74m to 0.56m while target visibility improves from from 82.8 to 86.4\%. 
% Maybe a second example like this highlighting the tradeoff/recoverability.
%the convergence rate in visual tracking mode, which represents the rate of decay of lyapunov function; the maximum recovery distance, which represents the unpper bound of the increase of lyapunov function in the unstable mode; and the average stable dwell time, affects tracking performance. Our results show that tracking performance improves with a higher tracking speed, a lower maximum recovery distance, or a longer average dwell time.
%}
% \yangge{Guided by the stability analysis, we investigate how design parameters—such as the tracking speed in nominal tracking mode, maximum recovery distance, and the average stable dwell time—affect tracking performance. Our results show that tracking performance improves with a higher convergence rate, a longer average dwell time, or a lower upper bound on the Lyapunov function's increase in the recovery mode.}
%
% various parameter settings, which highlighted its robustness and adaptability across different operational conditions, ensuring reliable performance. 
% even in challenging environments.
%
In summary, our contribution are as follows:
\begin{enumerate}
    \item The design of a Switched Visual Tracker (\vtc) for the visual tracking  problem, which effectively handles intermittent loss of target visibility.
    \item A rigorous stability analysis of the \vtc based on a modest extension of an existing  switched system stability result. This provides  guarantees on the tracking performance and guidelines for choosing  parameters in  the implementation of \vtc.
    \item An implemention of \vtc on an Agilicous-based pursuer drone and comprehensive experimental evaluations with different target trajectories and design parameters.
\end{enumerate}

\section{Visual Tracking: Problem \& Control Design}
\label{sec:method}
% \yangge{
% Section Plan
% \begin{enumerate}
%     \item Section 2.1, verbally talk about what the target tracking problem looks like. Talk about vision lose, and potentially two modes. 
%     \item Section 2.2, talk about two modes design, talk about differential equations in each modes
%     \item Section 2.3, talk about design of recovery mod
%     \item Section 3, connect back to section 2 at the beginning. We talk about general theorem about switching system stability not just the tracking system. COnnect to experiment. 
% \end{enumerate}
% }
%In this section, we describe the visual target tracking problem and our  solution approach. 

%\subsection{The Problem: Visual Target Tracking with Intermittent Loss of Sight}
%\label{sec:method:problem_description}

The \textit{visual target tracking} problem involves a  {\em target\/} and a {\em pursuer\/}, both moving in space. 
The pursuer has a camera with a limited field of view $\theta$ (Fig.~\ref{fig:recovery_demo}), and no prior knowledge of the target's trajectory. 
The goal is to design a {\em controller} that enables the pursuer to follow the target closely {\em and\/} maintain visual contact.  

\begin{figure}[h] 
    \centering
\includegraphics[width=0.9\linewidth]{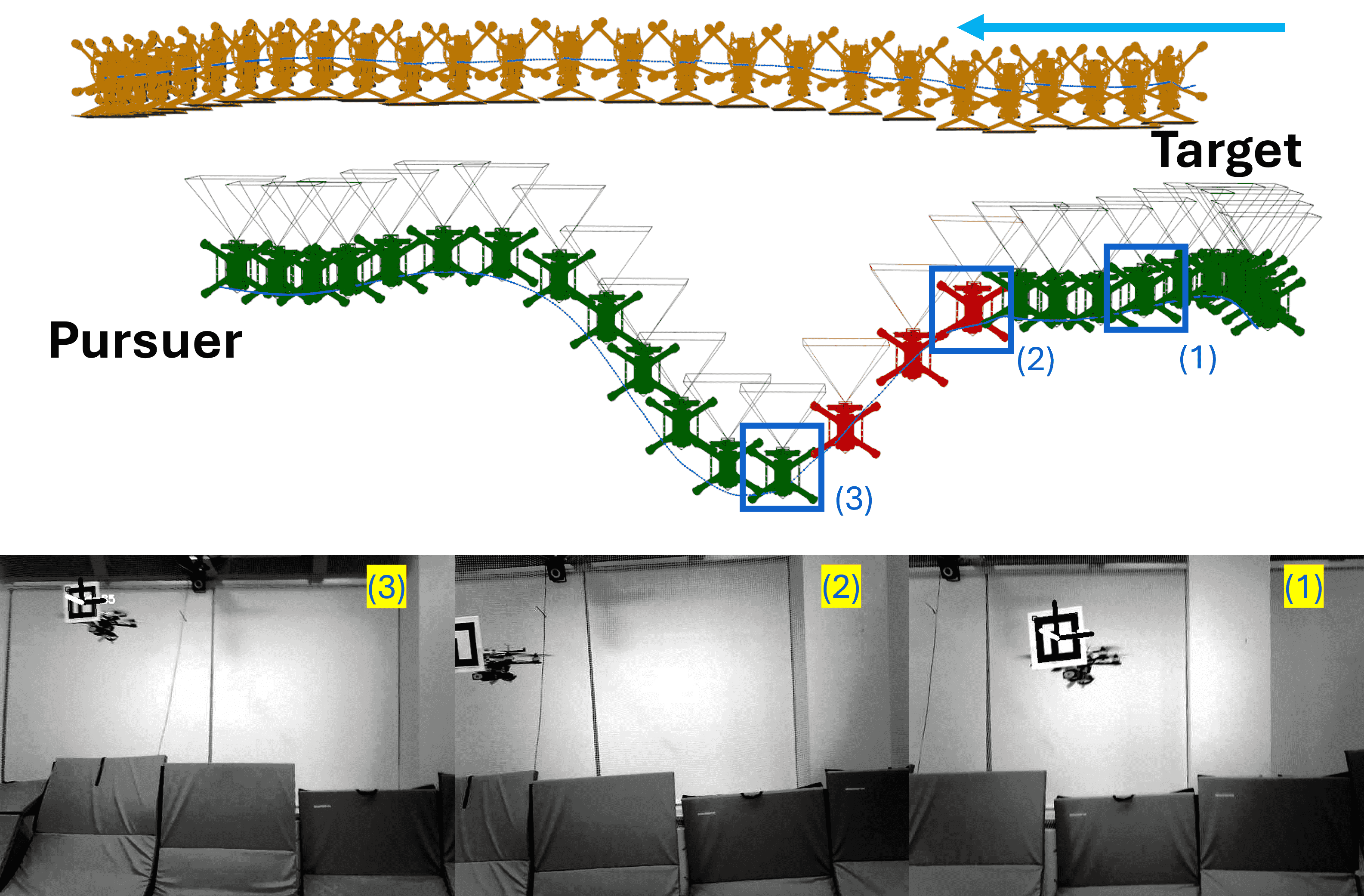}
    \caption{\small The target (orange) moves from right to left and is initially visible to the pursuer (green). When the target exits the pursuer's view (red), the pursuer performs a backoff maneuver to regain visibility. The pursuer's field of view (FoV) angle, 
    $\theta$, is represented by the triangle in front. The bottom row shows the camera images for three states of the pursuer. 
    }
    \label{fig:recovery_demo}
\end{figure}
% \sayan{Add a nice figure with 3 panels here that complements the text.}

The control design is challenging because there is tension between the two goals of maintaining visual contact and gaining proximity. 
The pursuer may intermittently lose sight of the target because of detection failures (e.g., motion blur, lighting) or the target maneuvering outside the pursuer's camera frame. When this happens, the pursuer has to regain visual contact by falling behind the target's {\em estimated position\/} and potentially getting further away from it. A simple instance of this is shown in Figure~\ref{fig:recovery_demo}. 

We tackle this challenge by designing a  {\em mode switching controller\/}~\cite{DL2003}: in the {\em visual tracking mode\/} the pursuer attempts to get close to the target, and in the {\em recovery mode} it maneuvers to regain visual contact. Next, we discuss the particulars of this design, including the mode switching logic. In Section~\ref{sec:thm}, we will provide a general result on stability analysis of such switching controllers and in Section~\ref{sec:experiments} we will discuss our implementation of the controller on quadrotors. 

\subsection{Design of  Switched Visual Tracking}
\label{sec:method:vttp_overview}

% The \vtc pipeline is composed by a observer, a planner and a low level controller. 
%Let $X$ be the state space for both the target and the pursuer agents. 
The position of the target ($T$) is modeled as a function of time or a {\em trajectory\/} in 3D-space 
$x_T^d: \mathbb{R}^{\geq 0} \rightarrow \mathbb{R}^3$.
That is,  $x_T^d(t)$ is the position of the target at time $t$.
% The motion of the target is defined by 
% \[
% \dot{x_T} = f_T(x_T, u_T)
% \]
%
% where $x_T\in X$ is the state of the target and $u_T\in U$ is control input to the target, and $f_T$ is the dynamics of the aerial vehicles \sayan{pointer?}. We model the target for the purpose of analyzing the controller, but, of course, $x_T$ and $u_T$ are unknown to the pursuer and the extent of knowledge available about $f_T$ only influences the conservativeness  of the recovery pose (see Section~\ref{sec:recover_pose}).
% \sayan{Lets point to some paper or appendix about the actual quadrotor dynamics which is used.}
%
The {\em Switched Visual Tracker (\vtc)\/} uses an {\em observer} that provides the relative position of the target to the pursuer ($P$), when the target (T) is visible. This observer can be implemented using fiducial markers~\cite{GARRIDOJURADO20142280}, perspective-n-point~\cite{9549863}, or various deep learning based approaches~\cite{s24041076,liu2024survey}, and our implementation is discussed in Section~\ref{sec:experiments:setup}.
%There are several vision-based  techniques for implementing such an observer~\cite{}.
% , and in our experiments we  use \sayan{name the algorithms used and cite}.  
%
% We abstractly model this camera-based observer of the \vtc as follows: the observer takes as input the state of both the pursuer and the target. If the target is in view then it outputs the  full state of the target $x_T$ 
% \sayan{or should it return the displacement vector $x_T^d - x_P^d$}, otherwise, it outputs a signal indicating that the target is not visible. 
We mathematically  model this camera-based observer as follows: 
the observer takes as input the positions of both the pursuer  and the target. 
If the target is in view,  
then the observer outputs the displacement vector $x_T^d-x_P^d$, otherwise, it outputs $\bot$ indicating that the target is not visible.
%
% We define visibility of target by the pursuer as 
% \begin{equation}
% \label{eqn:visibility}
% \begin{split}    
%     & visible (x_P, x_T;\theta ) \\ 
%     & = \begin{cases}
%         \top~\text{for}~V\cdot R>0, |V\times R| \leq (V\cdot R)tan(\theta)\\
%         \bot~\text{otherwise}
%     \end{cases}
% \end{split}
% \end{equation}
% where $V=x_T^d-x_P^d$ and $R$ is obtained by the applying the rotation matrix of the pursuer to a unit vector in z direction. 
We define visibility of target by the pursuer as a function $visible(x_P, x_T;\theta ) $ which returns True when $V\cdot R>0, |V\times R| \leq (V\cdot R)tan(\theta)$, where $V=x_T^d-x_P^d$ and $R$ is obtained by the applying the rotation matrix of the pursuer to a unit vector in z direction. It returns False otherwise. 
Our analysis accommodates estimation errors in $x_T^d - x_P^d$ but for the sake of simplifying the presentation we use the above observer  model. 

The visibility of the target determines the operating  mode of the \vtc, i.e. when the observer returns the full state of the target, the \vtc will operate in visual ($V$) tracking mode and when the observer returns $\bot$, the \vtc will operate in the recovery ($R$) mode. 

When \vtc is operating in visual tracking mode, the pursuer will follow dynamic equation 
\[
    \dot{x_P} = f_V(x_P, g(x^d_T-x^d_P))
\]
where $x_P$ is the full state of the pursuer. 
The function $g$ is a tracking controller which takes as input the displacement between the target and the pursuer $x^d_T-x^d_P$ and generates a control input that aims to drive the pursuer to a specified offset from the target. We use standard controller design techniques for creating such a tracking controller~\cite{9794477}. 
% \sayan{Introduce $\lambda$ here?}

When \vtc is operating in the recovery mode, the pursuer will perform maneuver to regain visual contact to the target. It will follow dynamics equation
\[
    \dot{x_P} = f_R(x_P, x_R),
\]
where $x_R$ is a {\em recovery pose\/}, from where the pursuer can see the target. The dynamics $f_R$ will drive the pursuer directly to the recovery pose. 

\subsection{Computing Recovery Pose}
\label{sec:recover_pose}
% , we first compute the set of states the target could be in using reachability analysis. Given that set of possible positions, we compute recovery pose. Then it move to that pose. The three steps should happen inside $t_R$

The recovery maneuver is parameterized by a time constant $t_R>0$ and  has three steps: 
(1) the controller uses {\em reachability analysis\/}~\cite{ARCH23:ARCH_COMP23_Category_Report} to predict a set $Reach(x_T(0),  t_{R})$ of positions where the target could possibly be located at $t_R$ starting from its last seen position $x_T(0)$;  
(2) it computes a pose $x_R$ from which $Reach(x_T(0),  t_{R})$  is visible, and (3) the pursuer moves to $x_R$.
Figure~\ref{fig:recovery_dist} shows the main idea. 

In more detail, reachability analysis is a well-established technology with many available tools that can compute $Reach(x_T(0),  t_{R})$  assuming that the target has a dynamic model 
$\dot{x_T} = f_T(x_T, u_T)$, 
with input (or disturbance) $u_T$ in some range  $U_T$.
Access to the actual dynamic model of the target is not required. 
 % when \vtc enters the recovery mode, it computes the target's reachable set $Reach(x_T(0), t_{R})$, which uses the knowledge of the approximate dynamic model $f_T$ and the input range $U_T$.  
%The actual dynamics used for our experiments are discussed in Section~\ref{sec:experiments:setup}. 
%
% The reachable set $Reach(x_T(0), t_{R})$ is represented as a sphere with center $c\in \mathbb{R}^3$ and a radius $r\in \mathbb{R}^+$. 
%With the computed reachable set from the last known state, 
The recovery point $x_R$ is computed so that $Reach(x_T(0), t_{R})$  falls within the pursuer's camera frame, which then is guaranteed to include the target in its viewing range.
% If the pursuer can perform maneuver such that the whole reachable set falls within its camera frame, then it is guaranteed to include the target in its viewing range. 
%To achieve this, as shown in

% given the field of view of the camera on the pursuer $\theta$, the recovery point $x_R$ will satisfy that 
% \begin{equation}
% \label{eqn:recovery_dist}
% ||x_R-Reach(x_T(0), t_{R}).c|| \leq \frac{Reach(x_T(0), t_{R}).r}{sin(\frac{\theta}{2})}
% \end{equation}
% In this case, as long as the pursuer can reach $x_R$ within the same time horizon $T$ used while computing the reachable set, it is guaranteed to include the target in its camera frame and regain vision contact. 

% where $Reach(\cdot).c$ and $Reach(\cdot).r$ are the center and radius of the reachable set. 
% Note that (\ref{eqn:recovery_dist}) allow many recovery pose and the exact one to use is per implementation. One implementation for computing the recovery point is described in Section~\ref{sec:experiments:setup}.

It not always  possible that the pursuer can reach the computed recovery point $x_R$ within the recovery time $t_R$, in which case the recovery maneuver would fail. We formalize the recoverability criteria as follows:

\begin{definition}
\label{def:recoverability}
From initial positions of the pursuer $x_P(0)$ and the target $x_T(0)$ with $visible(x_P(0),x_T(0),\theta)$, the system is  {\em $t_R$-recoverable\/}  if the computed recovery pose $x_R$ is such that 
(1) $x_T(t_R) \in Reach(x_T(0), t_R)$, and
(2) $x_P(0)$ can reach $x_R$ in time $t_R$.
% pursuer can \textit{recover} from a loss of vision if: 1) the reachable set $Reach_{f_T}$ is computed for $x_T(0)$, $U_T$, $f_t$, $t_R$ and contains actual state of target; 2) the recovery pose $x_R$ can be computed based on reachable set; 3) the pursuer drone can perform reachability analysis, compute the recovery pose and move to the recovery pose within the time constant $t_R$. 
\end{definition}

In Section~\ref{sec:thm}  we will assume recoverability for the analysis of \vtc and 
in Section~\ref{sec:experiments} we will empirically evaluate the recoverability for our particular experimental setup.

\begin{figure}[h]
    \centering    \includegraphics[width=0.6\linewidth]{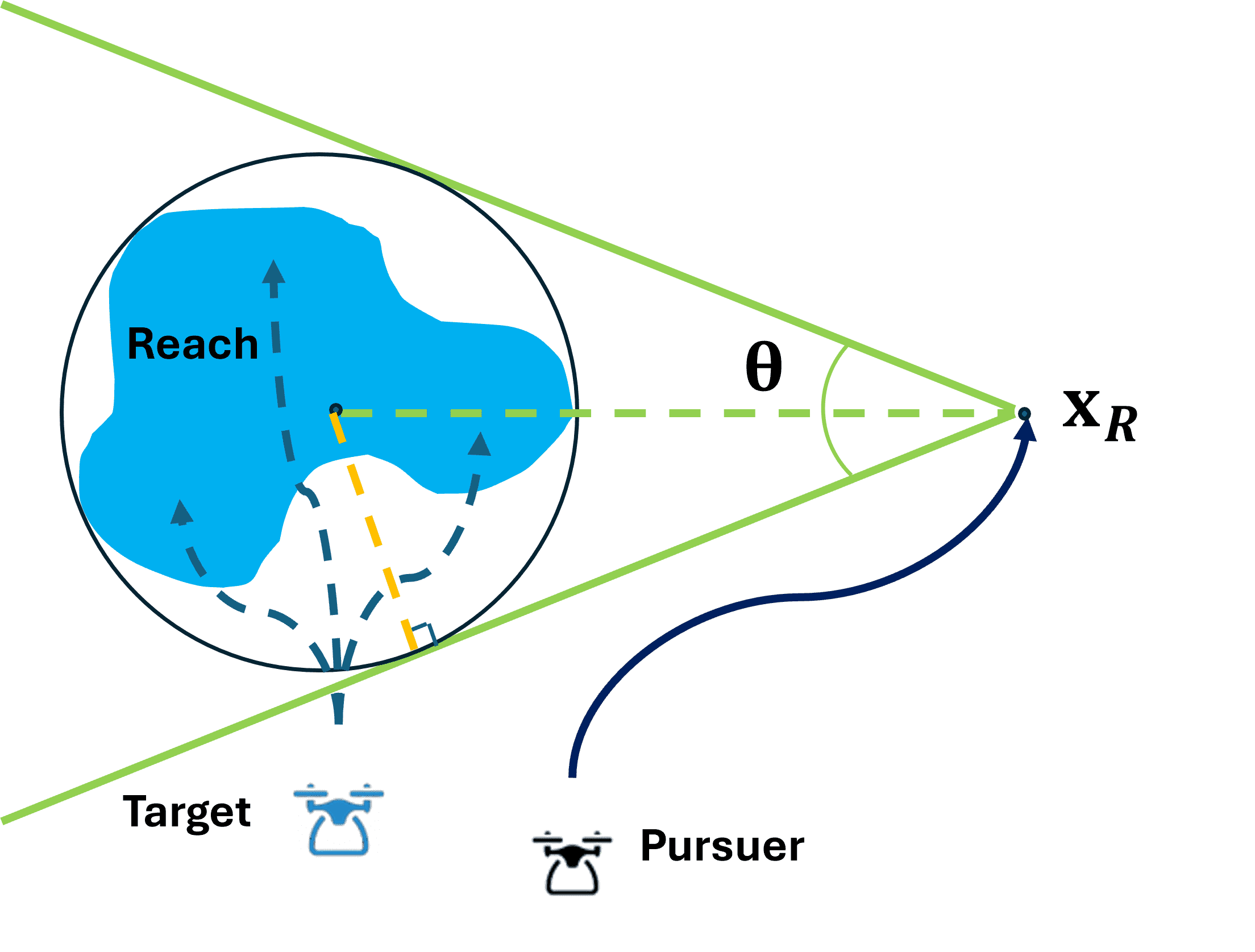}
    \caption{Computation of the recovery pose $x_R$. The pursuer (black) at $x_R$ can include the whole reachable set of target (blue) in it's camera range. 
    }
    \label{fig:recovery_dist}
\end{figure}

\section{Stability of Target Tracking System}
\label{sec:thm}
%\sayan{Give an overview of the section}
In this section, we will present a stability analysis of the \vtc. 
We can view this as a {\em  switched system\/}---a general class of dynamical systems in which an algorithm occasionally changes the plant dynamics. 
In \vtc, the two plant dynamics or {\em modes\/} are the the visual and the recovery modes and the switching decision is determined by the result of 
$visible(x_P, x_T;\theta)$.
% (\ref{eqn:visibility}). 
% \sayan{refer back to equations}. 
%
Our analysis is based on  Morse and Hespanha's Average dwell time theorem~\cite{hespanha99stability} which gives a stability criteria for switched systems. Roughly, that theorem (Theorem~4 of~\cite{hespanha99stability}) says that if the individual modes of the switched systems have  decaying  energy (or, more generally, Lyapunov functions $V_i$'s) and the increase in energy owing to mode switches is bounded, then the system is asymptotically stable if the rate of switching is not too high. The theorem gives the maximum allowable rate of mode switches in terms of an {\em average dwell time}, as a function of the rate of Lyapunov decay and the gains. Our Theorem~\ref{thm:hybrid_stability} below slightly generalizes that result by allowing some of the individual modes to be unstable for a bounded amount of time. This is necessary for analyzing \vtc because the tracking error can potentially increase in the recovery mode. We state Theorem~\ref{thm:hybrid_stability} in general terms because this may be of broader interest beyond the analysis of \vtc.

%an existing result to make it applicable to our target tracking problem. 

\subsection{Stability of Switching System}
A {\em switched system\/} with a finite set of {\em modes\/} $P$ is described by:
\begin{align}
  \dot{x}(t) = f_{\sigma(t)}(x(t)),
  \label{eq:ss}
\end{align}
where $x(t)$ is the plant state at time $t$, 
each $f_p:\mathbb{R}^n\rightarrow \mathbb{R}^n$ describes the dynamics in mode $p$, and
 $\sigma: [0, \infty) \rightarrow P$ is called a 
{\em switching signal} which chooses a particular mode  $p \in P$ for each time $t$. The switching signal $\sigma$ is a piece-wise constant\footnote{Strictly, its a c\`{a}dl\`{a}g function with $\sigma(t) = lim_{\tau\rightarrow t^+} \sigma(\tau)$, for each $t\geq 0$.} function that abstractly models the behavior of the mode switching logic as a function of time. 
The discontinuities in $\sigma$ are called {\em switching times\/}.

Fixing a switching signal $\sigma$ and an initial state $x_0$ uniquely defines a solution or an {\em execution\/} of the switched system which we denote by $x(t)$ (suppressing the dependence on $x_0$ and $\sigma$). We refer the reader to~\cite{DL2003} for the mathematical definition of solution of switched systems.
%which satisfies the differential equation (\ref{eq:ss}) for all $t$. 
%
An execution $x(t)$ of the switched system is {\em asymptotically stable\/} with respect to a set $X^*$\footnote{The reason why it's converging to $X^*$ instead of 0 will be discussed in Theorem~\ref{thm:hybrid_stability}.} if $x(t)$ converges to $X^*$ as $t \rightarrow \infty$.
%$\forall \epsilon>0$, $\exists T$ such that $\forall t\geq T, \xi(x_0, t)\in B_{\epsilon}(X^*)$. 
%
The switched system is {\em globally asymptotically stable\/} with respect to $X^*$ if $\forall x_0, \sigma$, all the executions are asymptotically stable. 
We say a mode $p$ is asymptotically stable if the subsystem $\dot{x} = f_p(x)$ is globally asymptotically stable. That is, if the switched system stays forever in mode $p$ then its solutions converge to 0.   
Let the $P_s \subseteq P$ be the set of  asymptotically stable modes. 

The standard method for proving stability of dynamical systems is to construct or  find a Lyapunov function as defined below.
\begin{definition}
\label{def:lyapunov}
A $\mathcal{C}^1$ function $V_p:\mathbb{R}^n\rightarrow \mathbb{R}$ is a {\em Lyapunov function\/} for mode $p$ if there exists two class $K_\infty$ functions $\alpha_1, \alpha_2$ and $\lambda>0$ such that 
\begin{equation}
\label{eqn:v_assum1}
\alpha_1(|x|)\leq V_p(x) \leq \alpha_2(|x|), \forall x, \forall p\in P, \mathit{and}
\end{equation}
%and, 
%there exist $\lambda>0$ such that 
\begin{equation}
\label{eqn:exp_lyapunov}
\frac{\partial V_p}{\partial x} f_p(x) \leq -2\lambda V_p(x), \forall x\neq 0, \forall p\in P_s.
\end{equation}
\end{definition}
Condition~(\ref{eqn:exp_lyapunov}) is slightly stronger than the usual $\frac{\partial V_p}{\partial x}f_p(x)\leq -W_p(x)$, for some positive definite $W_p$, but this difference is mild because it is known that there is no loss of generality in taking $W_p(x) = 2\lambda V_p(x)$ for some $\lambda >0$~\cite{DL2003}.

\subsection{Average Stable Dwell Time}

Let us denote the number of transitions that $\sigma$  makes in a time $[t, t']$  by $N_\sigma (t,t')$. 
%
% A switching signal $\sigma$ is said to have {\em dwell time\/} $\tau_d>0$, if the gap between any pair of consecutive switching times is at least $\tau_d$, i.e., $t_{i+1}-t_i\geq \tau_d$, for all $i$. 
A switching signal $\sigma$ has {\em average dwell time\/} $\tau_a > 0$ if there exists integer $N_0>0$ such that for any interval $[t,t']$
\[
N_\sigma(t,t') \leq N_0+\frac{t'-t}{\tau_a}.
\]
That is, such a  switching signal $\sigma$ allows for at most one mode switch in every $\tau_a$ and an additional  burst of $N_0$ switches.

% \subsection{Specific Type of Switching System}
% In this section, we discuss a specific type of switching system and the corresponding dwell time definition. 
Let the number of transitions to any asymptotically stable mode over the interval $[t,t']$ be denoted by $N_{\sigma s}(t,t')$, and the total amount of time spent in stable modes be  $T_{s}(t,t')$. 
A switching signal $\sigma$ has \textit{average stable dwell time} $\tau_{as}>0$, if there exists integer $N_0>0$ such that for any interval $[t,t']$ 
\begin{equation} 
\label{eqn:tau_a_def}
N_{\sigma s}(t,t') \leq N_0 + \frac{T_s(t,t')}{\tau_{as}}
\end{equation}
% \sayan{Rewrite this. Not clear what permissive means. For any $\sigma$ is the ASDT always <= ADT?}
% The definition of average stable dwell time is more permissive compared to average dwell time in the sense that the average stable dwell time only focus on $p\in P_s$ and impose less constraint on $p\in P\textbackslash P_s$.
Compared with $\tau_a$, $\tau_{as}$ imposes no constraint on the unstable modes. 
% \sayan{Is it ok to say if $\sigma$ has ADT $\tau_a$ then it also has ASDT $\tau_a$?}

% \yangge{}
% \sayan{If we are alternating between $P_s$ and $P\setminus P_s$, this explanation and notation could be simplified. Revisit and possibly merge with the previous section. (2) Clearly delineate old ADT definition and the new concept ASDT. Point out that ASDT is more permissive, right? Maybe you dont need DT or SDT, just ADT and ASDT. (3) You can assume w.lo.g that system starts in $P_s$ (or in $P\setminus P_s$) and simplify notations, etc.}

% The stable-unstable system is defined by a tuple $S = \langle X, P^\da, P^\ua, f_p, \sigma\rangle$. The system will always switch between stable and unstable modes, i.e., the all switching signals $p_0, p_1, p_2,...$ generated by $\sigma(t)$ satisfy that $p_{2k}\in P^\da$ and $p_{2k+1}\in P^\ua$ $\forall k\in \mathbb{N}$. The sequence of switching times is given by $t_1, t_2,...$ where $t_i$ is the switching time from mode $p_{i-1}$ to mode $p_i$. 

% For this system, a switching signal is said to have a stable dwell time $\tau_s$, if the gap between any pair of consecutive switching times for a stable mode is at least $\tau_s$, i.e., $t_{2k+1} - t_{2k}\geq \tau_d$, $\forall k$.

\subsection{Stability Analysis}
\label{sec:stability-analysis}
The visual target tracking system can be described by a class of switching system with a specific type of switching signal that alternating between visual tracking modes in $P_s$ and recovery modes in $P\textbackslash P_s$. Without loss of generality, assume $\sigma(0)\in P_s$, then all switching signals $p_0, p_1, p_2,...$ generated by $\sigma(t)$ satisfy that $p_{2k}\in P_s$ and $p_{2k+1}\in P\textbackslash P_s$, $\forall k\in \mathbb{N}$. The sequence of switching times is given by $t_1, t_2,...$ where $t_i$ is the switching time from mode $p_{i-1}$ to mode $p_i$. 
Here we provide stability analysis of this class of system.

\begin{theorem}
\label{thm:hybrid_stability}
Suppose we have a collection of Lyapunov functions $V_p$ $\forall p \in P_s$, and there exists $\mu>1, c>0$, such that for any even switch times $t_2,t_4,...$, 
% \sayan{simply say any even switching times?}
\begin{equation}   
\label{eqn:v_transition_increase}
V_{\sigma(t_i)}(x(t_i))\leq \mu V_{\sigma({t_{i-1}})}(x({t_{i-1}}))+c
\end{equation}
%
% Then, for all $\delta>0 $, if the average convergence dwell time satisfy that 
Then, for any $\delta>0$, for any switching signal $\sigma$ with $\tau_{as}$ satisfy
\begin{equation}
\label{eqn:tau_a_assu}
\tau_{as} > \frac{ln(\mu+\delta)}{2\lambda}    
\end{equation}
the
the system is asymptotically stable with respect to 
\begin{equation}
\label{eqn:convergence_set}
    X^*(\delta) = \left\{x\mid |x|\leq \alpha_1^{-1}\left(c\frac{(\mu+\delta)^{N_0}}{\delta}\right)\right\}
\end{equation}
% for every switching signal $\sigma$ the system will satisfy that as $t\rightarrow \infty$, 
% \[
% |x(T)|\leq \alpha_1^{-1}(c\frac{\mu+\delta}{\delta})
% \]
\end{theorem}

\begin{proof}
We provide a sketch of the proof here. From (\ref{eqn:exp_lyapunov}) and (\ref{eqn:v_transition_increase}), if we keep unrolling the lyapunov function value at any time $T$, we can get 
\begin{equation}
\label{eqn:ineqn_unroll}
\begin{split}
& V(T) \leq \mu^\frac{N}{2} V(0) e^{-2\lambda T_s(0)} + \sum_{k=0}^{\frac{N}{2}} c \mu^ke^{-2\lambda T_s(t_{N-2k})} 
\end{split}    
\end{equation}
where we shorten  $V_{\sigma(t)}(x(t))$ as $V(t)$ and $T_s(t, T)$ as $T_s(t)$. 
Using definition of $\tau_{as}$ (\ref{eqn:tau_a_def}), $\frac{N}{2} \leq N_0 + \frac{T_s(0)}{\tau_{as}}$ and (\ref{eqn:tau_a_assu}):
\[
\begin{split}
 \mu^\frac{N}{2} V(0) e^{-2\lambda T_s(0)}  
 \leq e^{\left(-2\lambda T_s(0)  \left(1-\frac{ln\mu}{ln(\mu+\delta)}\right)\right)} \mu^{N_0}V(0)
\end{split}
\]
which goes to 0 as $T$ goes to infinity.

Since in between time interval $[t_{N-2k}, T]$, there are $k+1$ switches from unstable mode to stable modes, according to (\ref{eqn:tau_a_def}), we know $T_s(t_{N-2k})\geq (k+1-N_0)\tau_{as}$ for all k. Then 
\[
\begin{split}
\sum_{k=0}^{\frac{N}{2}} c \mu^ke^{-2\lambda T_s(t_{N-2k})} 
 \leq c\left(1-\left(\frac{\mu}{\mu+\delta}\right)^{\frac{N}{2}}\right)\frac{(\mu+\delta)^{N_0}}{\delta}
\end{split}
\]
which converges to $c\frac{(\mu+\delta)^{N_0}}{\delta}$ as $T$ goes to infinity. This gives  
\[
|x(T)| \leq \alpha_1^{-1}(V(T)) \leq \alpha_1^{-1}\left(c\frac{(\mu+\delta)^{N_0}}{\delta}\right)
\]
\end{proof}

% For the visual tracking problem, the state $x$ can be the state difference between the target and the pursuer, which we want it to converge to a bounded set as stated in (\ref{eqn:convergence_set}). The lyapunov function $V_p$ can be the distance between the target and pursuer. In this case, the assumption (\ref{eqn:v_transition_increase}) describe the increase of distance between target and pursuer when \vtc is in recovery mode. When the pursuer is recoverable, the recovery state $x_R$ exists and the pursuer can move the $x_R$ with in $t_R$, which indicates that the distance increase between target and pursuer in the recovery mode is upper bounded and the assumption is automatically satisfied. 

In the visual tracking problem, the state $x$ represents the difference between the target and the pursuer, which converges to a bounded set as per (\ref{eqn:convergence_set}). The Lyapunov function $V_p$ is defined by the distance between them, and assumption (\ref{eqn:v_transition_increase}) describes the increase in this distance when the \vtc is in recovery mode. If the pursuer is recoverable, a recovery state $x_R$ exists, and the pursuer can reach $x_R$ within $t_R$, ensuring the distance increase is bounded and the assumption (\ref{eqn:v_transition_increase}) is satisfied.
The parameters in Theorem~\ref{thm:hybrid_stability} have physical meanings in this context: $\lambda$ represents the convergence rate of the pursuer in visual tracking mode, while $\mu$ and $c$ jointly define the maximum distance increase during recovery. In the next section, we demonstrate how this theorem guides the parameter design for the \vtc.

% In the visual tracking problem, the state $x$ represents the state difference between the target and the pursuer, which we aim to converge to a bounded set as per (\ref{eqn:convergence_set}). The Lyapunov function $V_p$ can be defined by the distance between them, and assumption  (\ref{eqn:v_transition_increase}) describes how this distance increases when \vtc is in recovery mode. 
% If the pursuer is recoverable, a recovery state $x_R$ exists, and the pursuer can reach $x_R$ within $t_R$, ensuring that the distance increase in recovery mode is bounded, thus satisfying the assumption.
% The parameters in Theorem~\ref{thm:hybrid_stability} also have their physics meaning for the visual tracking problem. $\lambda$ describe the nominal tracking speed of the pursuer in the nominal tracking mode. $\mu$ and $c$ jointly define the maximum recovery distance, i.e. the distance increase between the target and pursuer in recovery mode. 
% In the next section, we will see how the theorem guide us to design parameters in \vtc. 

%\newcommand{\sayan}[1]{\textcolor{blue}{#1}}
\section{Experimental Evaluation}
\label{sec:experiments}
We implemented \vtc  on an 
Agilicious drone~\cite{Foehn22science} and we will evaluate its performance in following various target trajectories of a second drone. We use {\em average tracking error (AE)\/} (i.e., the displacement  between the target and pursuer minus the {\em tracking offset}), and the {\em Fraction of time the target is visible (FTV)\/} from the pursuer as the two key performance metrics, which correspond to the proximity and the visibility requirements (recall the problem definition in Section~\ref{sec:method}). We study overall effectiveness of \vtc compared with a baseline visual tracking algorithm and how various parameters influence the  performance metrics. 

\subsection{Experimental Setup}
\label{sec:experiments:setup}
\paragraph*{Target Drone}
The target drone (Figure~\ref{fig:target_chaser_drone}) is built using a 6'' quadrotor frame. It carries a Raspberry Pi 3\textsuperscript{\textregistered} and Navio2\textsuperscript{\textregistered} for onboard computation. The target drone implements a trajectory tracking controller from~\cite{pmlr-v155-sun21b} and is capable of following a range of pre-scripted trajectories at different speeds. 

Experiments are carried out in a workspace of $5.6m\times 5.4m\times 3m$ volume. 
Although \vtc can track arbitrary trajectories, to create difficult, safe, and repeatable scenarios, we experiment with the target  moving parallel to the pursuer's image plane (This leads to frequent loss of sight and recoveries).
Here we discuss two target trajectories Ellip and SLem in the yz-plane with no movement in the x direction (Fig.~\ref{fig:LTs}) and three tracking offsets (1.0m, 1.5m and 2.0m), which provides 6 scenarios called: Ellip-1.0, SLem-1.5, etc.

\begin{figure}[h]
    \centering
    \includegraphics[width=0.95\linewidth]{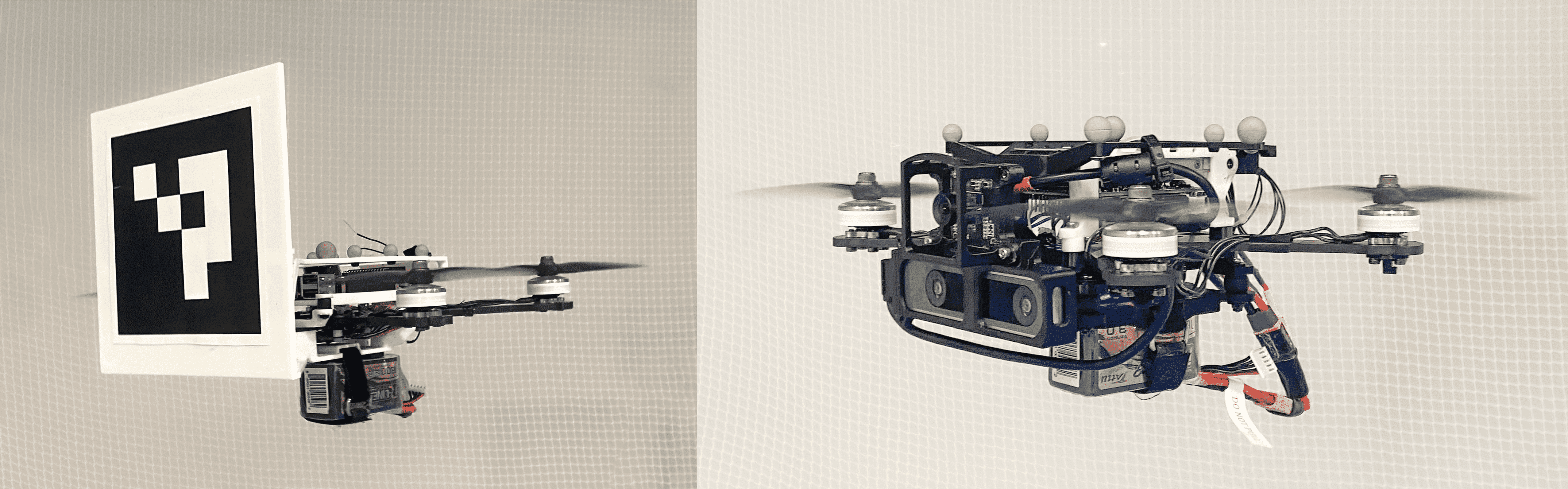}
    \caption{\small The target drone (left) and the pursuer (right).}
    \label{fig:target_chaser_drone}
\end{figure}

\begin{figure}[h]
    \centering
    \includegraphics[width=0.95\linewidth]{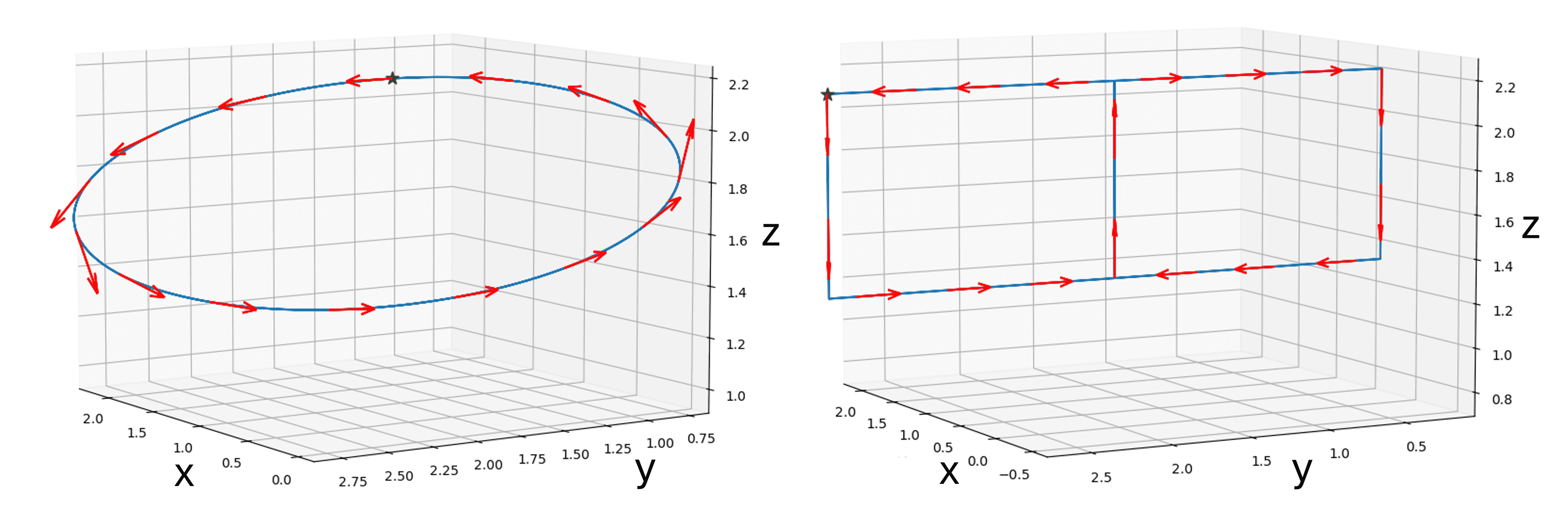}
    \caption{\small Sample of target trajectories Ellip (left) and SLem (right). 
    % Ellip have average velocity 0.572m/s and SLem have average velocity 0.600m/s.
    % \sayan{mark x,y,z coordinates clearly.} 
    }
    \label{fig:LTs}
\end{figure}

\paragraph*{Pursuer Drone}
The pursuer  (Figure~\ref{fig:target_chaser_drone}) is built based on Aigilicious platform~\cite{Foehn22science}, which uses a Nvidia Jetson TX2\textsuperscript{\textregistered} as the main computer, an Intel RealSense T265\textsuperscript{\textregistered} for localization, and an Arducam B0385 monocular camera for target detection. The Arducam can run at 100fps for 640x480 resolution images and has a horizontal field of view of $\theta = 70^\circ$. 
%TP is implemented on the chaser drone and all the computation happend on-board. 

% \sayan{Say something about Vicon for ground-truth.}
% The drones are operated in a $5.6m\times 5.4m\times 3m$ workspace. 

% \begin{figure}
%     \centering
%     \begin{minipage}{.3\textwidth}
%       \centering
%       \includegraphics[width=.4\linewidth]{figs/chaser_pic.jpg}
%       \caption{Chaser Drone}
%       \label{fig:test1}
%     \end{minipage}
%     \begin{minipage}{.3\textwidth}
%       \centering
%       \includegraphics[width=.4\linewidth]{figs/leader_pic.jpg}
%       \caption{Leader Drone}
%       \label{fig:test2}
%     \end{minipage}
% \end{figure}

%\subsubsection{The workspace and coordinate frame} 

% \sayan{These details seem premature}
% We initialize the chaser drone such that it's local frame is aligend with the world frame, i.e., positive x is moving forward, positive y is moving leftward and positive z is moving upward. The cameras on the chaser drone is pointing at positive x direction. 

%\paragraph{\vtc implementation details} 
%The \vtc is implemented on the pursuer drone.
\vtc runs entirely on the Jetson. 
%Based on our assumption in %Section~\ref{sec:method:vttp_overview}, 
The observer in \vtc uses an ArUco marker~\cite{GARRIDOJURADO20142280} on the target and a detection algorithm~\cite{opencv_library} running on the pursuer's camera images which gives a relatively accurate  estimate of the relative position of the target, when it is visible. A Kalman filter is used to smooth target's position and velocity estimates.

An MPC controller~\cite{9794477} is used by \vtc as the tracking controller $g$ in the visual tracking mode and to drive the pursuer to $x_R$ in the recovery mode. 

% For the low level controller $g$, we use the MPC controller implemented provided by the Agilicious drone~\cite{9794477}. 

To compute the recovery pose $x_R$, we assume the target follows a double integrator dynamics, with acceleration in  $\pm2 m/s^2$ range.
Reachability analysis is performed using Verse~\cite{10.1007/978-3-031-37706-8_18} to compute $Reach(x_T(0), t_R)$. 
We validated empirically that with these parameters and $t_R=1.5s$ the system is $t_R$-recoverable (Definition~\ref{def:recoverability}). A  general study of recoverability gets into pursuer-evader games and is beyond the scope of this work.
Three parameters in the implementation of the \vtc are significant: 
%pursuer, the visual convergence rate, the \mbd and $\tau_{as}$. We explore how they influence the overall tracking performance. 
(1) The tracking speed (\maxv) of the pursuer in visual tracking mode, which determines $\lambda$ in Theorem~\ref{thm:hybrid_stability}, is set by limiting the pursuer's x velocity. 
% This mainly influence how fast the pursuer can recover from the recovery pose after exiting the recovery mode. 
(2) The maximum recovery distance (\mbd) measure the maximum change in x dimension for the pursuer in recovery mode over each run of scenario. The value represent the increase of lyapunov function in $P\textbackslash P_s$ in Theorem~\ref{thm:hybrid_stability}. 
(3) The average stable dwell time ($\tau_{as}$) is one of the essential parameter in Theorem~\ref{thm:hybrid_stability} for controlling the tracking performance. 
We let $N_0=1$ and $\tau_{as}$ is measured by 
$\tau_{as} = \frac{T_s(0,T_{max})}{k-1}$
where k is the number of switches from recovery mode to visual tracking mode in a scenario run. 
In this set of experiment, $\tau_{as}$ reflects how long pursuer have vision of target ($T_s$) and how frequent the pursuer lose visual contact of target ($k$). 

Note that we don't have control over the \mbd and $\tau_{as}$. However, we can measure these values empirically and observe their effect on the  tracking performance. 

\subsection{Effectiveness of Tracking Algorithm}
We first compare  \vtc against a baseline tracking algorithm, which only follows the target when it is in sight. 
We run all 6 scenarios with both algorithms with \maxv ($\lambda$)  1m/s. For each algorithm, we run each scenario 4 times, take the average of the metrics, and report the  results in Table~\ref{tab:tracking}. 

The results show that \vtc achieves a better fraction of time visible in all scenarios and lower average tracking error in all but SLem-1.0. 
In Ellip-1.0, the average tracking error between the target and pursuer is reduced up to 45\%. 
% and in SLem-1.0, the fraction of time the target remains visible is improved by 8.3$\times$. 
This clearly indicate that  the \vtc helps the pursuer track the target more effectively than the baseline.
Fig.~\ref{fig:exp_run} % exp_08-21-14-02, exp_09-03-19-29
illustrates example runs for Ellip-1.0, where the baseline fails to maintain tracking while \vtc successfully follows the target's motion.
Moreover, as tracking distance decreases and the leader trajectory becomes more complex, the improvement with \vtc becomes more obvious.

% The trajectory of two example runs for Ellip-1.0 is shown in Fig.~\ref{fig:exp_run}. % exp_08-21-14-02, exp_09-03-19-29
% From the figure we can also observe that for baseline, the pursuer is not able to keep track of the target while running \vtc, the pursuer is able to follow the motion of the target. 
\begin{figure}[h]
    \centering
    \includegraphics[width=0.9\linewidth]{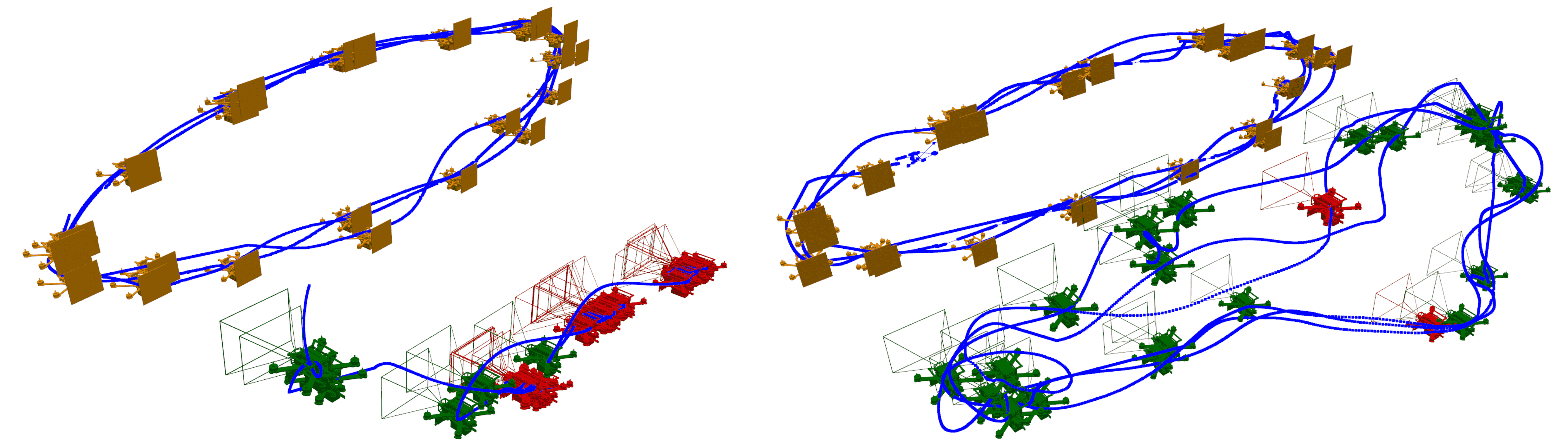} 
    \caption{An example run of Ellip-1.5 using basline (left) and \vtc (right). The pusuer is marked as red when the target is not visible.}
    \label{fig:exp_run}
\end{figure}

\begin{table}[h]
\caption{Tracking Performance with \vtc and baseline. Each data point obtained by average of 4 runs.}
\label{tab:tracking}
\begin{center}
\begin{tabular}{c c c c c}
\hline
\scenario & AE \vtc & AE Base& FTV \vtc & FTV Base\\
\hline
Ellip-1.0 &0.54& 0.98 & 85.8\%& 25.3\%\\
Ellip-1.5 &0.42 & 0.56 & 93.9\%&  52.0\%\\
Ellip-2.0 &0.36 & 0.34 & 98.7\%&  84.7\%\\
\hline
SLem-1.0 & 1.80 & 1.29 & 43.0\% & 5.2\% \\ 
SLem-1.5 &0.80 & 0.84 & 65.5\%& 27.2\%\\ 
SLem-2.0 &0.33 & 0.33 & 96.2\%& 88.3\%\\ 
\hline
\end{tabular}
\end{center}
\end{table}

\subsection{Tracking Rate and  Performance Tradeoff}
We explore the effect of \maxv on the tracking performance.
As Theorem~\ref{thm:hybrid_stability} indicates, reducing \maxv lowers $\lambda$, causing the system to converge more slowly to the set $X^*$. Consequently, $\tau_{as}$ must increase to maintain tracking performance or it will degrade.
% According to Theorem~\ref{thm:hybrid_stability}, as the \maxv decrease, the $\lambda$ in the theorem decrease. As $\lambda$ decrease, the whole system will converge slower towards the set $X^*$ and $\tau_{as}$ will have to increase or the tracking performance will degrade. 

Following this guidance, we ran scenarios Ellip-1.0 and SLem-1.5 with varying \maxv. Each parameter setting was tested 4 times, and the metrics were averaged. The results in Table~\ref{tab:lambda}
show that as \maxv decreases, average tracking error increases, while the fraction of time visible remains largely unaffected, confirming our hypothesis.
% 
% Under this guidance, We run scenario Ellip-1.0 and SLem-1.5 with different \maxv.  For each parameter we run each scenario 4 times and take the average of the metrics. 
% The result is shown in Table~\ref{tab:lambda}. From the results, we can observe that as \maxv decrease, average tracking error increase while fraction of time visible is less influenced. 

However, in practice, \maxv cannot be made arbitrarily large. Our experiments show that excessively high \maxv can cause movement overshoot and rapid pitch changes in the pursuer, which can reduce tracking performance. 
% This is consistent with Theorem~\ref{thm:hybrid_stability}, because as \maxv decrease, the $\lambda$ in the theorem decrease. As $\lambda$ decrease, longer $\tau_{as}$ is needed to achieve the same level of tracking performance. However, as indicated by FTV, the $\tau_{as}$ of the system is not changed, which leads to the decline in tracking performance.  
% In this case, the radius of the set that the system is converging to will increase. However, changing $\lambda$ not necessarily influence $\tau_{as}$, and therefore don't have much influence on FTV. 
% , the tracking performance also decrease. 
% \sayan{This is consistent with 
% Theorem~\ref{thm:hybrid_stability}, because ... }

\begin{table}[h]
\caption{\small Visual Convergence Rate on Tracking Performance. Each data point obtained by average of 4 runs.}
\label{tab:lambda}
\begin{center}
\begin{tabular}{c c c c}
\hline
\scenario &\maxv & AE & FTV\\
\hline
Ellip-1.0 & 0.1 & 1.27 & 88.3\%\\
Ellip-1.0 & 0.5 & 0.66 & 84.4\%\\
Ellip-1.0 & 1.0 & 0.54 & 85.8\%\\
\hline 
SLem-1.5 & 0.1 & 1.25 & 58.4\% \\ 
SLem-1.5 & 0.5 & 1.07 & 47.7\%\\ 
SLem-1.5 & 1.0 & 0.81 & 65.5\%\\ 
\hline
\end{tabular}
\end{center}
\end{table}
% \sayan{Again, multiple metrics? Some of these table are small and could be doubled-up. I wonder if it will be better to present some of the tables as graphs.}

\subsection{Recoverability-Performance Tradeoff}
% We examine how the tracking performance is influenced by \mbd. The \mbd describes the maximum increase of the lyapunov function (\ref{eqn:v_transition_increase}) in Theorem~\ref{thm:hybrid_stability}, especially the constant increase $c$. As $c$ increases, the radius of the set to which the system converges also increases, leading to a degradation in tracking performance. 

We examine how tracking performance is affected by \mbd, which represents the maximum increase of the Lyapunov function (\ref{eqn:v_transition_increase}) in Theorem~\ref{thm:hybrid_stability}, particularly the constant $c$. As $c$ increases, the convergence radius grows, resulting in degraded tracking performance.

% In this experiments, we look at scenario Ellip-1.0 and Ellip-1.5 and run the scenario with different \maxv. For each scenario, we look at how \mbd influence the tracking performance. The result is shown in Table~\ref{tab:c}. 
% From the data, we observe that as maximum recover distance increase, average tracking error increase and fraction of time visible slightly decrease, which means as this parameter increase, the tracking performance becomes worse. 
% However, we can't set $c$ to arbitrary small as if we constraint maximum recover distance to be too small, finding $x_R$ within that range maybe unfeasible. In this case, the pursuer can be unrecoverable from a vision loss. 

In these experiments, we consider scenarios Ellip-1.0 and Ellip-1.5, running each with different \maxv to examine how \mbd affects tracking performance. The results are shown in Table~\ref{tab:c}. We observe that as the maximum recovery distance increases, the average tracking error rises and the fraction of time the target is visible slightly decreases, indicating worsened tracking performance. However, \mbd cannot be set arbitrarily small; if the maximum recovery distance is too limited, finding $x_R$
within that range may become unfeasible, making the pursuer unable to recover from a vision loss.

\begin{table}[h]
\caption{\small Maximum Recovery distance on Tracking Performance}
\label{tab:c}
\begin{center}
\begin{tabular}{ c c c c c }
\hline
\scenario &\maxv & \mbd& AE & FTV\\
\hline
Ellip-1.5 & 1.0 & 1.23 & 0.43 & 91.8\%\\
Ellip-1.5 & 1.0 & 1.84 & 0.52 & 87.8\%\\
Ellip-1.5 & 1.0 & 1.85 & 0.51 & 89.7\%\\
\hline
Ellip-1.0 & 0.5 & 2.13 & 0.56 & 86.4\%\\
Ellip-1.0 & 0.5 & 2.43 & 0.66 & 83.6\%\\
Ellip-1.0 & 0.5 & 2.70 & 0.74 & 82.8\%\\ 
\hline
Ellip-1.0 & 1.0 & 1.29 & 0.48 & 87.4\%\\ 
Ellip-1.0 & 1.0 & 1.65 & 0.52 & 85.2\%\\ 
Ellip-1.0 & 1.0 & 2.06 & 0.63 & 84.5\%\\
\hline
\end{tabular}
\end{center}
\end{table}

Examining the change in \mbd, we note that, as outlined in Section~\ref{sec:recover_pose}, the recovery pose is computed from the $Reach(x_T(0),t_R)$, which depends on the target's initial state and the time constant $t_R$. As $t_R$ increases, the reachable set expands, placing the recovery pose farther away and making it harder for the pursuer to reach. However, $t_R$ cannot be too small; if it is, the pursuer may fail to reach the recovery pose in time due to computation and hardware delays, violating the assumptions in Definition~\ref{def:recoverability}.

\subsection{Stable Dwell Time and  Performance Tradeoff}

We look into how the tracking performance is influenced by $\tau_{as}$. We keep \maxv as 1m/s, with the leader drone following Ellip. 
% We look at cases where the \mbd is almost the same to reduce the effect of it on the tracking performance. 
To isolate the impact of $\tau_{as}$ , we use cases where \mbd is nearly constant.
The result is shown is Table~\ref{tab:asdt}.

\begin{table}[h]
    \caption{$\tau_{as}$ Tracking Performance}
\label{tab:asdt}
\begin{center}
\begin{tabular}{ c c c c c }
\hline
\scenario & \mbd& $\tau_{as}$& AE& FTV\\
\hline
Ellip-2.0 & 1.63 & 44.27 & 0.407 & 98.4\%\\
Ellip-1.5 & 1.69 & 9.88 & 0.412 & 94.2\%\\ 
Ellip-1.0 & 1.65 & 3.12 & 0.515 & 85.2\%\\
Ellip-1.0 & 1.67 & 3.15 & 0.548 & 86.1\%\\
\hline 
Ellip-1.5 & 1.21 & 12.86 & 0.411 & 95.4\%\\ 
Ellip-1.5 & 1.34 & 9.57 & 0.412 & 93.9\%\\ 
Ellip-1.5 & 1.23 & 6.17 & 0.426 & 91.8\%\\ 
Ellip-1.0 & 1.29 & 3.36 & 0.479 & 87.4\%\\ 
\hline
\end{tabular}
\end{center}
\end{table}

The results show that as 
$\tau_{as}$ increases, the average tracking error decreases and the fraction of time visible increases, indicating improved overall tracking performance. 
A higher $\tau_{as}$ means the pursuer loses sight of the target less frequently or maintains visual contact for longer, naturally leading to a higher fraction of time visible. 
This aligns with the intuition that the system performs better with fewer losses of visual contact.

\section{Related Works}
\label{sec:related}
% In~\cite{7963823}, the author present a vision-based target tracking and autonomous landing system on a ground vehicle. Their approach utilizes a bag-of-visual-words model with SIFT features for object detection, combined with a feedback linearization controller that defines error in both the world and image frames. This setup ensures that a slow-moving target remains within the camera's field of view. However, the method does not provide a mechanism for recovering from a complete loss of vision of the target.

% Numerous studies have been conducted in the field of vision-based tracking. In this section, we discuss several key works that are particularly relevant to our research.

A  challenging version of the  visual  tracking problem in unknown, cluttered environments is studied in~\cite{9196703,7759092,9561948}. The system developed in~\cite{9561948}  uses trajectory prediction and a  planner that combines kinodynamic search with spatio-temporal optimization to maintain safe tracking trajectories. The prediction errors arising from  temporary occlusions are mitigated through high-frequency re-planning. While the  experimental results are compelling,  these approaches are not analytically tractable and don't have any mechanisms for handling longer-term loss of the target. In contrast, our work explicitly addresses the challenge of vision loss by incorporating a  recovery mechanism and provides stability analysis with recovery. 

% In~\cite{Gomez-Balderas2013}, the author solve the problem of perform tracking on a moving target on the ground. The target's movement is estimated using optical flow, and a switching controller is employed to handle both nominal tracking and loss of visual contact with the moving target. The proposed recovery strategy involves adjusting the altitude of the UAV until the target is brought back into the camera's field of view. However, this method lacks a mechanism for predicting the target's future movements and does not provide any formal guarantees on successfully recovering visual contact.

The problem of tracking multiple targets in a simply-connected 2D environment is studied in~\cite{5980372}. The approach uses a probabilistic model to generate sample trajectories that represent the likely movements of evader and construct a path for the purser that optimize the expected time to capture while ensuring guaranteed time to capture doesn't significantly increase.  The proposed recovery strategy involves systematically clearing the entire environment to ensure all regions are visible. However, this method requires the environment to be bounded and known in advance, whereas our approach does not impose specific constraints on the environment.

\section{Limitations and Future Directions}
\label{sec:limits}
We presented the design of  a Switched Visual Tracker (\vtc) that switches  between a visual tracking and a recovery modes to effectively track moving targets that can  becomes intermittently invisible. We analyzed the tracking performance of the system by extending the average dwell time theorem from switched systems theory. The proposed \vtc is implemented on hardware, and its effectiveness is demonstrated through several experiments.

%On the other hand, our approach has several limitations that open up potential directions for future research.  
One limitation of the current work comes from the  recoverability assumption (Definition~\ref{def:recoverability}). For our \vtc to be recoverable, we assumed the existence of a recovery point, $x_R$, and that both the computation of $x_R$ and the movement of the pursuer to $x_R$ occur within the specified time constant $t_R$. The computation of $x_R$ used reachability analysis and did not rely on sophisticated trajectory predictions. Developing  sufficient conditions for formally checking recoverability for different kinds of  target trajectories, predictions, and pursuer dynamics would broaden the applicability of $\vtc$.

Another future direction is to extend the application of  \vtc to more complex environments with obstacles. Consequently, computing the recovery pose will need path planning around  obstacles, adding an additional layer of challenge in the design and analysis.


%%%%%%%%%%%%%%%%%%%%%%%%%%%%%%%%%%%%%%%%%%%%%%%%%%%%%%%%%%%%%%%%%%%%%%%%%%%%%%%%

%%%%%%%%%%%%%%%%%%%%%%%%%%%%%%%%%%%%%%%%%%%%%%%%%%%%%%%%%%%%%%%%%%%%%%%%%%%%%%%%

%%%%%%%%%%%%%%%%%%%%%%%%%%%%%%%%%%%%%%%%%%%%%%%%%%%%%%%%%%%%%%%%%%%%%%%%%%%%%%%%
% \input{sec-appendix}


% \section*{ACKNOWLEDGMENT}

% The preferred spelling of the word ÒacknowledgmentÓ in America is without an ÒeÓ after the ÒgÓ. Avoid the stilted expression, ÒOne of us (R. B. G.) thanks . . .Ó  Instead, try ÒR. B. G. thanksÓ. Put sponsor acknowledgments in the unnumbered footnote on the first page.

%%%%%%%%%%%%%%%%%%%%%%%%%%%%%%%%%%%%%%%%%%%%%%%%%%%%%%%%%%%%%%%%%%%%%%%%%%%%%%%%
% \balance
\bibliographystyle{IEEEtran}
\bibliography{egbib,sayan1}
\end{document}